\documentclass[sigconf]{acmart}
\AtBeginDocument{%
  }

\usepackage{hyperref}
\usepackage{url}
\usepackage[dvipsnames]{xcolor}
\usepackage{xspace}
\usepackage{amsmath, amsthm}
\usepackage{enumitem}
\usepackage[capitalize,noabbrev,nameinlink]{cleveref}
\usepackage{graphicx}
\usepackage[table]{xcolor}
\usepackage{makecell}
\usepackage{multirow}
\usepackage{wrapfig}
\usepackage{adjustbox}
\usepackage{rotating}
\usepackage{lscape}
\usepackage{algpseudocode}
\usepackage{algorithm2e}
\usepackage{subcaption}
\usepackage{pifont}
\usepackage{booktabs}

\usepackage{amsmath,amsfonts,bm}

\def\eqref#1{equation~\ref{#1}}

\def\1{\bm{1}}

\def\vt{{\bm{t}}}

\def\vw{{\bm{w}}}
\def\vx{{\bm{x}}}

\def\vz{{\bm{z}}}

\def\mA{{\bm{A}}}

\def\mD{{\bm{D}}}

\def\mI{{\bm{I}}}

\def\mL{{\bm{L}}}

\def\mP{{\bm{P}}}

\def\mS{{\bm{S}}}
\def\mT{{\bm{T}}}

\def\mW{{\bm{W}}}
\def\mX{{\bm{X}}}

\def\mZ{{\bm{Z}}}

\DeclareMathAlphabet{\mathsfit}{\encodingdefault}{\sfdefault}{m}{sl}
\SetMathAlphabet{\mathsfit}{bold}{\encodingdefault}{\sfdefault}{bx}{n}

\def\gG{{\mathcal{G}}}

\newcommand{\E}{\mathbb{E}}

\newcommand{\R}{\mathbb{R}}

\newcommand{\mypar}[1]{\textbf{#1}\hspace{4pt}}

\usepackage{acronym}
\acrodef{GNN}[GNN]{Graph neural network}
\acrodef{GCN}[GCN]{graph convolutional network}
\acrodef{MLP}[MLP]{multi-layer perceptron}
\acrodef{GFM}[GFM]{graph foundation model}
\acrodef{GT}[GT]{graph transformer}
\acrodef{NLP}[NLP]{natural language processing}
\acrodef{LLM}[LLM]{large language model}
\acrodef{CV}[CV]{computer vision}
\acrodefplural{MI}[MI]{Mutual information}
\acrodefplural{KL}[KL]{Kullback-Leibler}
\acrodef{RNN}[RNN]{recurrent neural network}
\acrodefplural{SSL}[SSL]{self-supervised learning}
\acrodefplural{GRL}[GRL]{graph representation learning}

\newcommand{\jm}[1]{\textcolor{black}{#1}}
\newcommand{\dooho}[1]{\textcolor{black}{#1}}

\setcopyright{acmlicensed}
\copyrightyear{2026}
\acmYear{2026}
\setcopyright{cc}
\setcctype{by-nc-nd}
\acmConference[KDD 2026] {Proceedings of the 32nd ACM SIGKDD Conference on Knowledge Discovery and Data Mining V.2}{August 9--13, 2026}{Jeju Island, Republic of Korea.}
\acmBooktitle{Proceedings of the 32nd ACM SIGKDD Conference on Knowledge Discovery and Data Mining V.2 (KDD 2026), August 9--13, 2026, Jeju Island, Republic of Korea}
\acmISBN{979-8-4007-2259-2/2026/08}
\acmDOI{10.1145/3770855.3817828}
\begin{document}

\title{Node4All: Learning Node Representation Beyond Datasets}

\author{Dooho Lee}
\affiliation{%
  \institution{KAIST}
  \city{Daejeon}
  \country{Republic of Korea}
}
\email{dooho@kaist.ac.kr}

\author{Jaemin Yoo}
\affiliation{%
  \institution{Seoul National University}
  \city{Seoul}
  \country{Republic of Korea}
}
\email{jaeminyoo@snu.ac.kr}


\begin{abstract}
Node representation learning has advanced rapidly, yet most existing methods rely on per-dataset training and hyperparameter tuning.
This dataset-specific optimization comes from the difficulty of designing reusable graph models that generalize across diverse graph datasets.
In this work, we introduce Node4All, a node representation learner \emph{applicable to arbitrary graph datasets without any dataset-specific optimization}.

Node4All is built on two complementary ideas.
At the architectural level, we introduce the Channel Graph Transformer (CGT), which enables a single fixed parameterization to process arbitrary graph datasets.
At the learning level, we propose a self-supervised learning based on a series of synthetic graphs. 
Together, these components enable generalization beyond individual datasets, which is infeasible with existing architectures and learning frameworks.
We extensively evaluate Node4All on node classification across 25 benchmarks against 21 baselines, covering both supervised and self-supervised methods. 
Despite all baselines being trained and optimized for each dataset, a single Node4All, applied uniformly across the datasets, achieves a competitive ranking of 5th among 21 baselines.
Moreover, Node4All supports one-shot and in-context learning with an appropriate predictor and outperforms recent graph foundation models (GFMs) in these settings. 
These results demonstrate that Node4All not only achieves reusability across arbitrary graph datasets, but also remains an effective solution in practice.
Code and model checkpoints are available in \href{https://github.com/dooho00/node4all}{\textcolor{blue}{{this link}}}.
\end{abstract}


\maketitle
\section{Introduction}
\label{sec:introduction}


\jm{A graph is a data structure consisting of a set of nodes $V$ and a set of edges $E$ that connect pairs of nodes.
Graphs provide a natural way to describe many real-world systems by representing entities and their relationships, such as social networks, molecules, knowledge graphs, and biological interactions \cite{snapnets, hu2020open, wang2017knowledge}.}
At the same time, this flexibility makes graphs difficult to process directly: different graphs can have different numbers of nodes $N = |V|$, and their connectivity patterns in $E$ can vary widely, complicating the design of a single processing framework that works across graphs.

\begin{figure}
\centering
\includegraphics[width=\linewidth]{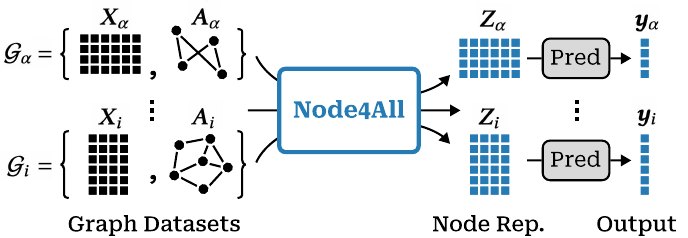}
    \vspace*{-5mm}
    \caption{\dooho{Overview. Node4All is a pretrained node representation learner that generalizes across arbitrary graph datasets.}}
    \Description{Node4All}
    \label{fig:overview}
    \vspace*{-5mm}
\end{figure}

Graph representation learning (GRL) aims to address this challenge by mapping a graph to a set of compact node representations \(\{\vz_i\}_{i=1}^N\) in a shared latent space $\vz_i\in\mathbb{R}^H$ where $H$ denotes the representation dimensionality~\cite{chen2020graph, ju2024comprehensive, khoshraftar2024survey}. 
These representations provide a standardized interface of graph-structured data and are commonly used as inputs to downstream predictors. 
As their quality directly influences downstream performance, learning node representations that effectively capture both structural and feature information has remained a central research problem for over a decade.

Early approaches fall into two main paradigms. The first factorizes proximity matrices derived from graph connectivity, including Laplacian Eigenmaps~\cite{anderson1985eigenvalues}, LLE~\cite{roweis2000nonlinear}, and HOPE~\cite{ou2016asymmetric}. The second learns node embeddings from random-walk co-occurrences using skip-gram--style objectives, including DeepWalk~\cite{perozzi2014deepwalk}, LINE~\cite{tang2015line}, and node2vec~\cite{grover2016node2vec}. 
While both paradigms are simple and effective, these methods learn embeddings tied to node identities and therefore require full re-optimization when new nodes are introduced, making the learned representations \emph{node-bound}.

\acp{GNN} address this limitation by learning shared message-passing functions that compute node representations from node features and local graph structure, rather than from node identities~\cite{defferrard2016convolutional, hamilton2017inductive, xupowerful, velikovi2018graph, wu2019simplifying, brody2021attentive}.
This design enables inductive application to unseen nodes and, with advanced architectures and tuned hyperparameters, \acp{GNN} achieve strong empirical performance across diverse benchmarks~\cite{luo2024classic, luocan, lee2025aggregation}.
However, \acp{GNN} are typically trained in supervised settings tied to a specific label space, making the learned representations inherently \emph{label-bound} and necessitating re-optimization for each new label set.

Recently, \acp{SSL} on graphs has been actively explored.
SSL-based methods remove the reliance on labels by learning representations through training objectives defined directly on the input graph.
Contrastive methods encourage agreement between representations obtained from different augmented views of the same node or graph, such as BGRL~\cite{thakoor2021bootstrapped}, MVGRL~\cite{hassani2020contrastive}, and CCA-SSG~\cite{zhang2021canonical}.
Generative approaches learn representations by reconstructing masked or perturbed graph structures and node features, including GAE~\cite{kipf2016variational}, GraphMAE~\cite{hou2022graphmae}, and MaskGAE~\cite{li2023s}.
The resulting representations are reusable across label spaces and often outperform supervised \acp{GNN}.

Despite these significant advances, most graph representation learning methods, from early random-walk--based approaches to recent \acp{SSL} methods, remain largely \emph{dataset-bound}. 
They are trained separately on each target dataset and require dataset-specific architectural choices and hyperparameter tuning to achieve strong performance. 
Furthermore, models trained on one dataset are not typically applicable to graphs with different feature specifications or topologies. 
As a result, from a user’s perspective, graph learning requires repeating training and tuning for every new dataset, a process that is time-consuming, expertise-intensive, and often requiring labeled data for reliable validation.

In this work, we propose \emph{Node4All}, a single node representation learner that can be applied to arbitrary graph datasets, as illustrated in \Cref{fig:overview}.
Node4All imposes no restrictions on graph topology or feature specification and eliminates the need for dataset-specific optimization, including both training and hyperparameter tuning.
With Node4All, users can obtain reusable node representations via a single inference.
Node4All is composed of two key components that address complementary aspects of the problem:
\begin{enumerate}[label=\textbf{\arabic*.}, leftmargin=1.5em]
\item \textbf{Architecture.} We propose \emph{Channel Graph Transformer (CGT)}, a novel architecture that can process arbitrary graphs under a single fixed parameterization without any specific tuning.
\item \textbf{Learning.} We propose a self-supervised learning approach on synthetic graphs, which are carefully designed to mimic massive real-world graphs, to enable generalization beyond datasets.
\end{enumerate}
\noindent
We evaluate Node4All on node classification across 25 benchmarks spanning diverse feature specifications and structural properties.
We compare against 21 baselines, including both supervised and self-supervised methods, each extensively optimized for individual datasets.
Although Node4All is \emph{(i)} trained without observing any real-world datasets and \emph{(ii)} applied uniformly across all benchmarks without any modification, 
it achieves a competitive ranking of 5th among 21 baselines. 
Moreover, our experiments further show that with an appropriate predictor, Node4All supports one-shot and in-context learning, outperforming recent graph foundation models (GFMs) in these settings.
These results demonstrate that Node4All not only achieves reusability across arbitrary graph datasets, but also generalizes well and remains effective in practice.
\section{Preliminaries}

\jm{We introduce the notations and background on graph representation learning and \acp{GNN} used throughout the paper.}

\mypar{Notations.}
Let $\gG = (\mathcal{V}, \mathcal{E})$ be an undirected graph with $N = |\mathcal{V}|$ nodes. Its adjacency matrix is $\mA \in \{0,1\}^{N \times N}$, where $\mA_{ij}=1$ if and only if $(i,j)\in \mathcal{E}$. Node features are given by $\mX \in \R^{N \times F}$, 
\jm{where each row $\mX_{n,:} \in \R^F$ represents the feature vector of the $n$-th node with dimensionality $F$, and each column $\mX_{:,f} \in \R^N$ represent the $f$-th feature channel across all $N$ nodes.}
Without loss of generality, we represent the graph as $\gG=(\mX,\mA)$ throughout this work.

\mypar{Graph Representation Learning.}
Given a graph $\gG=(\mX,\mA)$ with $N$ nodes, we formalize graph representation learning as learning a function $\Psi$ that maps the graph to $H$-dimensional node representations $\{\vz_i\}_{i=1}^N$, collected as a matrix $\mZ=\Psi(\mX,\mA)\in\mathbb{R}^{N\times H}$:
\begin{equation*}
\Psi: \mathbb{R}^{N\times F} \times \{0,1\}^{N\times N} \rightarrow \mathbb{R}^{N\times H}.
\label{eq:psi}
\end{equation*}
Our goal is to learn a representation function $\Psi$ that can process graphs with arbitrary numbers of nodes ($\forall N$) and feature dimensions ($\forall F$) within a single fixed parameterization. The resulting embeddings $\mZ$ should jointly encode features and graph structure, serving as sufficient node representations for downstream tasks.

\mypar{Graph Neural Networks.}
Given a graph $\gG=(\mX,\mA)$ with $N$ nodes, a graph neural network (GNN) iteratively updates node representations $\mZ^{(l)} \in \R^{N \times F_l}$, starting from $\mZ^{(0)}=\mX$, where $F_l$ denotes the representation dimensionality at layer $l$.
Each layer can be written as a message-passing function $\phi^{(l)}$ \citep{MessagePassing, xu2018jknet, pre-trainGNN},
\begin{equation}
\mZ^{(l)}=\phi^{(l)}(\mZ^{(l-1)},\mA)=
\mathrm{Comb}^{(l)} \Big(\mathrm{Agg}^{(l)}(\mZ^{(l-1)},\mA),\,\mZ^{(l-1)}\Big),
\label{eq:mp}
\end{equation}
where $\mathrm{Agg}^{(l)}$ aggregates the neighboring representations using a matrix $\hat{\mA} \in \R^{N \times N}$ derived from adjacency matrix $\mA$, either statically~\citep{kipf2017semisupervised, hamilton2017inductive} or dynamically~\citep{velikovi2018graph, chienadaptive}, and $\mathrm{Comb}^{(l)}$ integrates the aggregated messages with the previous representations, typically via a learnable linear transformation $\mW^{(l)} \in \R^{F_{l-1} \times F_l}$.

\section{Limitations of Existing Works}
\label{subsec:limitation_arch}

We analyze three common assumptions underlying most existing graph representation learning methods, which limit their applicability and generalization beyond individual datasets.

\mypar{L1. Fixed Feature Dimensionality.}
Most existing methods instantiate the representation function $\Psi$ using \acp{GNN}. These models encode information through feature transformations parameterized by weight matrices $\mW \in \mathbb{R}^{F \times H}$, where $F$ denotes the input feature dimensionality and $H$ the hidden dimension. This design inherently couples the model parameterization to a fixed input dimensionality $F$. In practice, however, graph datasets exhibit substantial variation in feature dimensionality, ranging from only a few dimensions~\cite{platonov2023a} to several thousand~\cite{shchur2018pitfalls}.
Consequently, models trained on a specific feature dimensionality cannot be directly transferred to datasets with different dimensionalities $\tilde{F} \neq F$, preventing a single trained model from being applied across datasets.

\mypar{L2. Optimized on a Specific Dataset.}
Most graph representation learning methods are trained and optimized on a specific target dataset. This dataset-centric optimization ties the learned representations to dataset-specific regularities, including both structural patterns and feature distributions.
As a result, such models often struggle to generalize to nodes or graphs that are rare or differ from those observed during training~\cite{yehudai2021local, mao2023demystifying, li2025out}. While optimization on a target dataset is a natural and effective choice for in-dataset performance, it inherently limits the reuse of learned representations across arbitrary graph datasets.

\mypar{L3. Specialized to the Target Task.}
In supervised learning, representations are updated to separate nodes with different labels while collapsing nodes with the same label. When transferred to a different task with a new label definition, such representations may no longer align with the new semantics, leading to degraded performance. 
Although \acp{SSL} removes explicit label supervision, many methods are not fully task-independent in practice.
In particular, contrastive approaches rely on the assumption that differently augmented views share the same semantic meaning.
\jm{This assumption may not hold depending on the task and dataset,} and in practice effective augmentations are usually identified through validation performance on the target task. As a result, this indirect guidance ties the learned representations to that task, limiting their generalization to different tasks and datasets.

Together, we identify the obstacles that prevent existing methods from generalizing beyond individual datasets from three aspects of the learning framework: what they train \textbf{L1}, where they train \textbf{L2}, and how they train \textbf{L3}.
In Sections \ref{sec:processing} and \ref{sec:learning}, we introduce the components of \emph{Node4All} that address each of these challenges.

\section{Processing Arbitrary Graphs}
\label{sec:processing}

In this section, we introduce channel-wise graph processing and its realization, the \emph{Channel Graph Transformer (CGT)}, which can process graphs with arbitrary feature dimensionality under a single parameterization, thereby addressing \textbf{L1}.

\subsection{Channel-wise Graph Processing}
\label{subsec:cip}

Our key idea is to treat node features as a collection of independent feature channels rather than as a single vector. By processing each channel with a shared model, the architecture naturally supports arbitrary feature dimensionalities under a single parameterization.
To formalize this, we define a \emph{channel graph} as a graph induced by a single feature channel with shared node and edge sets.

\begin{definition}[Channel Graph]
Let $\gG=(\mX, \mA)$ be a graph whose node features consist of $F$ scalar signals
$\mX=\{\mX_{:,c}=\vx^{(c)}\in\mathbb{R}^{N}\}_{c=1}^{F}$.
The \emph{channel graph} induced by channel $c$ is defined as
\[
\gG^{(c)} = (\vx^{(c)}, \mA), \quad c=1,\dots,F.
\]
\end{definition}

\noindent
From this perspective, the graph $\gG$ can be regarded as a collection of channel graphs
$\{\gG^{(c)}\}_{c=1}^{F}$ that share the same topology.
This viewpoint naturally leads to \emph{channel-wise graph processing}, where the same model is applied uniformly to each channel graph.
\begin{definition}[Channel-wise Graph Processing]
Given a graph $\gG=\{\gG^{(c)}\}_{c=1}^{F}$,
\emph{channel-wise graph processing} defines a representation function $\Psi$ by applying a shared model $f$ to each channel graph and concatenating the results along feature dimension:
\begin{align*}
\mZ = \Psi(\mX,\mA)
&\triangleq
\big[
f(\gG^{(1)}) \;\|\; \cdots \;\|\; f(\gG^{(F)})
\big] \\
&=
\big[
\vz^{(1)} \;\|\; \cdots \;\|\; \vz^{(F)}
\big]
\in \mathbb{R}^{N\times F},
\end{align*}
where the shared model
$f:\mathbb{R}^{N}\times\{0,1\}^{N\times N}\rightarrow\mathbb{R}^{N}$
maps each channel graph $\gG^{(c)}$ to a scalar node embedding
$\vz^{(c)}\in\mathbb{R}^{N}$.
\end{definition}

\noindent
This formulation allows $f$ to process graphs with arbitrary feature dimensionality $F$, thereby addressing \textbf{L1}. We next discuss which architecture should be used to realize $f$.

\begin{figure}[t]
\centering
\includegraphics[width=\linewidth]{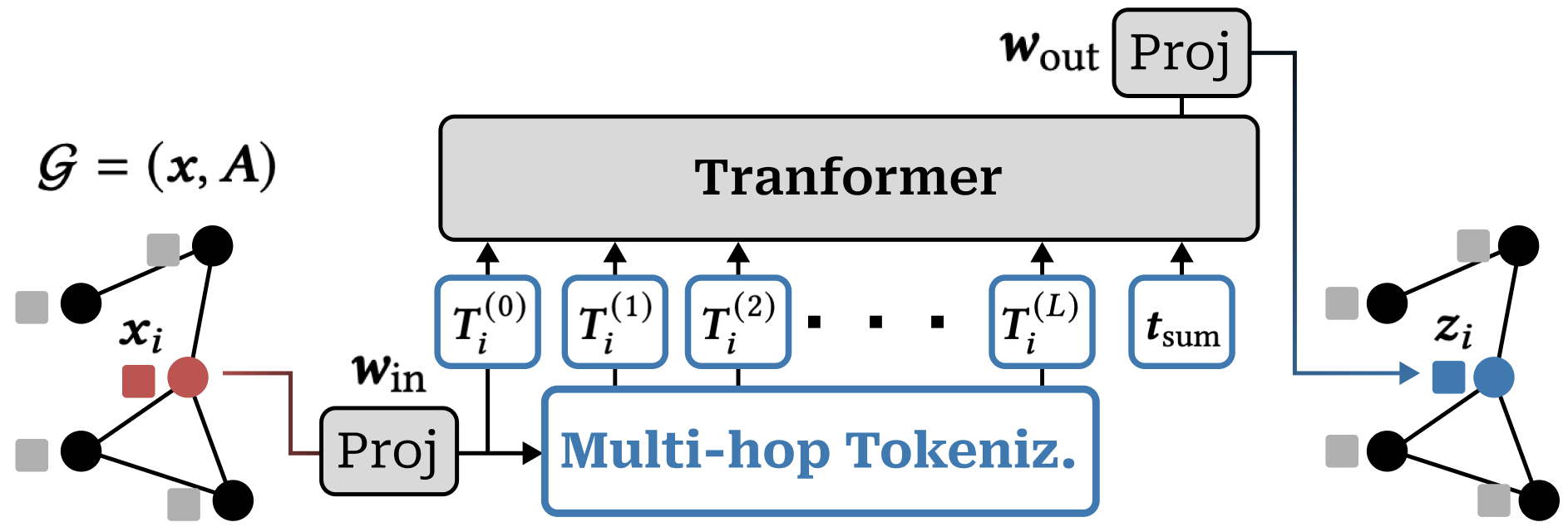}
\vspace{-4mm}
\caption{The Channel Graph Transformer (CGT) maps each channel graph to scalar node embeddings via multi-hop tokenization and Transformer-based token readout.}
\vspace{-4mm}
\Description{Illustration of CGT}
\label{fig:cgt}
\end{figure}

\subsection{Channel Graph Transformer (CGT)}
\label{sec:cgt}

A standard \ac{GNN} is suboptimal for processing channel graphs, as its representational power largely arises from parametric feature transformations that mix multiple feature channels.
This capacity disappears if a channel graph is given as input, and \acp{GNN} degenerate into repeated propagation of a one-dimensional signal. 

To address this limitation, we propose the Channel Graph Transformer (CGT), an architecture specialized for channel graphs. 
As shown in \Cref{fig:cgt}, 
CGT \emph{(i)} encodes each node in the channel graph, represented by a scalar feature, into a series of multi-hop tokens, and \emph{(ii)} uses a Transformer~\cite{vaswani2017attention} to aggregate the tokens into an updated scalar embedding.
This design provides substantially greater expressive capacity for learning on channel graphs than \acp{GNN}.

\subsubsection{Multi-hop Tokenization}
\label{subsec:cgt:tokenize}

A Transformer requires tokens with a fixed dimensionality that carry meaningful information on their own, such as image patches for computer vision~\cite{dosovitskiy2021an} or word tokens for language~\cite{vaswani2017attention}.
In graphs, however, a node feature alone does not encode structural information, since such information is defined through interactions with its neighborhood.
CGT constructs standardized and structurally informative tokens through \emph{coordinate-wise multi-hop aggregation}.

Given a channel graph $\gG=(\vx,\mA)$, we first map the scalar node signal to an initial $d$-dimensional token
\[
\mT^{(0)}=\vx\vw_{\mathrm{in}}\in\mathbb{R}^{N\times d},
\qquad
\vw_{\mathrm{in}}\in\mathbb{R}^{1\times d}.
\]
We then assign each coordinate of the token a distinct receptive-field size.
At each CGT layer $l$, the $i$-th coordinate is propagated by exactly $i$ hops over the graph, and the resulting representations are combined to form a new $d$-dimensional token:
\begin{equation*}
\mT^{(l)}=
g^{(l)}
\Big(
\big[
\hat{\mA}^{0}\mT^{(l-1)}_{:,0}\,\Vert\,
\hat{\mA}^{1}\mT^{(l-1)}_{:,1}\,\Vert\,
\cdots\,\Vert\,
\hat{\mA}^{d-1}\mT^{(l-1)}_{:,d-1}
\big]
\Big)\in\mathbb{R}^{N\times d},
\label{eq:multi_aggregation}
\end{equation*}
where $\hat{\mA}$ denotes the symmetrically normalized adjacency matrix.
The mixing function $g^{(l)}:\mathbb{R}^{N\times d}\rightarrow\mathbb{R}^{N\times d}$ is implemented using a lightweight \ac{MLP} for simplicity.

Our tokenization method ensures that each token has a consistent dimensionality $d$ while embedding rich neighborhood information of each node from local to global scales. 
Repeating this for $L$ layers yields a series of tokens ${\mT^{(0)},\mT^{(1)},\dots,\mT^{(L)}}$, which serve as the input sequence to the Transformer.

\subsubsection{Transformer-based Token Readout}
\label{subsec:cgt:readout}

CGT then aggregates the token sequence of each node into a single representation using a Transformer. To perform this aggregation, we prepend a learnable summary token $\vt_{\mathrm{sum}}\in\mathbb{R}^{d}$ to the token sequence of each node $i$:
\begin{equation*}
\mS_i^{(0)} =
\big[
\vt_{\mathrm{sum}},\;
\mT^{(0)}_i,\;
\mT^{(1)}_i,\;
\cdots,\;
\mT^{(L)}_i
\big]
\in \mathbb{R}^{(L+2)\times d},
\quad i=1,\dots,N.
\label{eq:cgt_token_seq}
\end{equation*}
We then apply an $R$-layer Transformer to this sequence. Through self-attention, it attends over the entire token set and aggregates the information into $\vt_{\mathrm{sum}}$.
We take the output representation of $\vt_{\mathrm{sum}}$ at layer $R$ and project it to a scalar embedding:
\[
f(\gG)_i = \mS_{i,0}^{(R)\top}\vw_{\mathrm{out}}\in\mathbb{R},
\qquad
\vw_{\mathrm{out}}\in\mathbb{R}^{d}.
\]
Together, CGT realizes the function $f$ required for channel-wise graph processing, while providing sufficient expressive capacity to capture structural information in the channel graph.
\section{Learning for Arbitrary Graphs}
\label{sec:learning}

While our CGT enables the processing of arbitrary graphs, learning representations that generalize across datasets remains challenging, which relates to limitations \textbf{L2} and \textbf{L3}. 
To address these limitations, we propose a new \acp{SSL} framework based on synthetic graphs.

\begin{figure}[t]
\centering
\includegraphics[width=\linewidth]{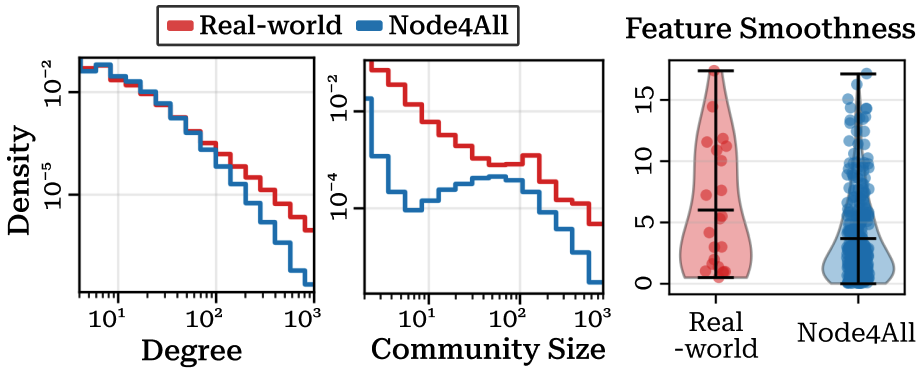}
\vspace{-6mm}
\caption{
Comparison between 25 real-world graph datasets and 300 Node4All synthetic graphs.
Degree and community-size (identified by the Louvain algorithm) distributions show similar trends between real and synthetic graphs.
Per-dataset average feature smoothness also spans a comparable range between synthetic and real-world datasets.
}
\label{fig:dataset}
\Description{Autoencoder}
\vspace{-4mm}
\end{figure}

\subsection{Synthetic Graph Generation}
\label{subsec:synthetic}


Most graph models are trained on only one or a small number of real-world datasets and are fitted to dataset-specific patterns (\textbf{L2}).
In the language and vision domains, this issue is often mitigated by training a model on massive collections of real-world datasets. 
However, achieving such large-scale coverage is difficult for graph models due to the lack of semantically aligned datasets.

To address this challenge, we train the model on synthetic graphs instead of limited real-world datasets, exposing it to a wide range of graph distributions and encouraging it to focus on structural and feature patterns that persist across datasets.
Given that many possible design choices are available for synthetic graph generation, we keep the design simple to avoid introducing unnecessary inductive biases while covering a wide range of graph distributions.


\subsubsection{Overview}
We define a synthetic graph generator as a probability distribution over graphs $\gG=(\mX,\mA)$, denoted by $p(\gG)$ and parameterized by latent variables $\vz$:
\begin{equation*}
p(\gG)
=
\int p(\vz)\,p(\mA \mid \vz)\,p(\mX \mid \mA,\vz)\,d\vz.
\label{eq:synthetic_graph_model}
\end{equation*}
A graph $\gG \sim p(\gG)$ is generated by a three-stage process:
(i) sampling latent variables $\vz$,
(ii) generating the graph structure $\mA$ conditioned on $\vz$, and
(iii) generating node features $\mX$ conditioned on both $\mA$ and $\vz$.
We next describe the structure and feature generators, as well as the latent parameters that control them.

\subsubsection{Structure Generation}
\label{subsec:structure_generation}

We generate graph structures using a Chung--Lu--style random graph model~\cite{chung2002connected}.
The structure generator is controlled by two graph-level latent variables: the expected average degree $\bar d$ and a degree bias parameter $\sigma$, which determines the degree heterogeneity across nodes.

Given the number of nodes $N$ and parameters $(\bar d,\sigma)$, we assign each node $u\in\{1,\dots,N\}$ an expected degree $w_u$ via
\[
\alpha_u \sim \mathcal{N}\!\left(-\tfrac{1}{2}\sigma^2,\sigma^2\right),
\qquad
w_u = \bar d \exp(\alpha_u).
\]
This approach preserves the expected average degree across nodes, i.e., $\mathbb{E}[w_u] = \bar d$, while introducing heterogeneity of individual nodes. 
When $\sigma\approx 0$, the graph is close to regular, whereas larger $\sigma$ yields increasingly heavy-tailed degree distributions.

Given the sampled expected degrees $\{w_u\}_{u=1}^N$, we then sample an undirected adjacency matrix $\mA \in \{0,1\}^{N\times N}$ by independently drawing each unordered node pair $1\le u<v\le N$ with
\[
A_{uv} \sim \mathrm{Bernoulli}(p_{uv}),
\quad
p_{uv} = \min\!\left(1,\frac{w_u w_v}{W}\right),
\quad
W = \sum_{t=1}^N w_t.
\]
Under this construction, the realized degree $d_u$ of each node satisfies
$\E[d_u \mid \{w\}] \approx w_u$.
Consequently, the graph-level average degree concentrates around
$\frac{1}{N}\sum_{u=1}^N d_u \approx \bar d$. 
Overall, this procedure generates diverse graph structures while controlling sparsity through $\bar d$ and degree heterogeneity through $\sigma$, yielding degree distributions similar to those of real-world graph datasets (see \Cref{fig:dataset}).

\subsubsection{Feature Generation}
\label{subsec:feature_generation}

Our feature generation pipeline is inspired by recent tabular foundation models~\cite{hollmann2022tabpfn}, but is adapted to graphs in two key ways:
(i) by explicitly controlling the intrinsic dimensionality of the feature matrix, and
(ii) by modeling correlations between node features and graph structure through propagation.

Controlling intrinsic dimensionality is motivated by the observation that many real-world datasets lie on low-dimensional semantic manifolds despite being represented in high-dimensional feature spaces.
This property is particularly important for graphs, where lower intrinsic dimensionality increases the relative contribution of structural information in downstream tasks.

To capture this effect, we introduce a rank parameter $\alpha \in (0,1]$, which determines the number of semantic factors $\bar F=\lceil \alpha F\rceil$ for a given feature dimension $F$.
We sample a semantic matrix $\mS\in\mathbb{R}^{N\times\bar F}$ and map it to the feature space via linear mixing,
\begin{equation*}
\mX=\mS\mW\in\mathbb{R}^{N\times F},
\qquad
\mW\in\mathbb{R}^{\bar F\times F}.
\end{equation*}
We initialize $\mW$ with orthonormal rows so that the effective rank~\cite{roy2007effective} of $\mX$ is approximately $\bar F$, thereby controlling intrinsic dimensionality through $\alpha$.
Feature diversity arises from sampling $\mS$ and $\mW$, while varying $\alpha$ induces diversity in intrinsic dimensionality, allowing the generator to produce a range of realistic low-rank feature structures observed in real-world datasets.

However, since these features are generated independently of the graph structure, they exhibit no inherent correlation with the graph.
In particular, as shown theoretically in Appendix~\ref{appendix:lemma}, random assignment of node features degenerates graph-aligned variations in expectation.
To counteract this effect, we apply stochastic propagation that explicitly amplifies graph-aligned variations:
\begin{equation*}
\mX \leftarrow \Bigl(\bigl(1+\beta\alpha\bigr)\mI-\alpha,\hat{\mA}\Bigr)^S \mX,
\quad
\beta\in{0,1},\ \alpha\in[0,1],
\end{equation*}
where $\hat{\mA}$ denotes the symmetrically normalized adjacency matrix.
The operator is controlled by three latent variables: the propagation depth $S$, the mixing strength $\alpha$, and a mode selector $\beta$.
When $\beta=0$, the operator reduces to $\mI - \alpha \hat{\mA}$, emphasizing differences between a node and its neighbors.
When $\beta=1$, it becomes $\mI + \alpha(\mI - \hat{\mA})$, which amplifies these differences while retaining the original node signal.
By sampling $S$, $\alpha$, and $\beta$, the propagation introduces diverse patterns of graph-aligned \jm{features.}
As a result, the generated features span a range of smoothness comparable to that observed in real-world graph datasets, as shown in \Cref{fig:dataset}.

\begin{figure}[t]
\centering
\vspace{-2mm}
\includegraphics[width=\linewidth]{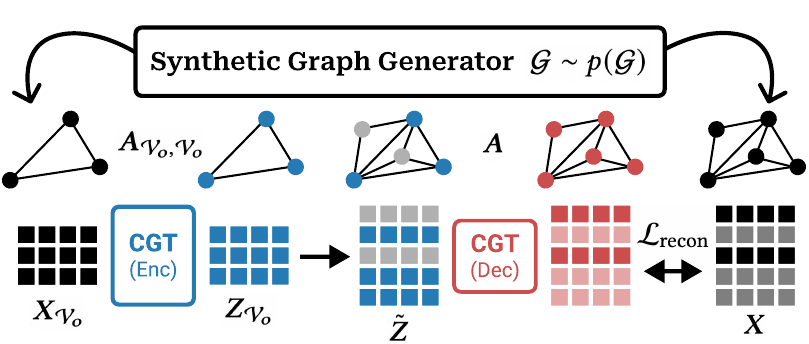}
\vspace{-6mm}
\caption{Node4All trains CGT via masked autoencoding on a stream of synthetic graphs to learn node representations that generalize to arbitrary real-world graph datasets.}
\label{fig:node4all}
\Description{Autoencoder}
\vspace{-4mm}
\end{figure}

\subsection{Masked Autoencoder with DropNode}
\label{subsec:mae_dropnode}

Having specified what to train (CGT) and where to train it (synthetic graphs), we now describe how we train it. 
We adopt \emph{masked node-feature reconstruction}, following GraphMAE~\cite{hou2022graphmae, hou2023graphmae2}, which learns representations by reconstructing masked node features.

This \jm{self-supervised objective} relies on a weak and broadly applicable assumption: node attributes are at least partially predictable from their surrounding context. 
If this assumption were violated, it would imply that nodes in the graph are largely independent of one another, calling into question whether the graph structure encodes meaningful relational information at all. 
Such scenarios are uncommon in practical graph datasets, making masked reconstruction a robust  objective that mitigates task-specific bias \jm{(\textbf{L3}).}

However, directly applying GraphMAE is not ideal in our setting. Unlike prior work that assumes a single dataset with a fixed feature space, our training data consists of multiple synthetic graphs, each with potentially different feature spaces.
Thus, replacing masked nodes with a single shared learnable token, e.g., \texttt{[MASK]}, as done in GraphMAE, can lead to inconsistent semantics across graphs and introduce misleading shortcuts during training.

We therefore adopt a DropNode-style~\cite{you2020graph, feng2020graph} masked autoencoder, which drops nodes instead of replacing them with a mask token.
Given a \jm{synthetic} graph $\gG=(\mX,\mA)$ with a node set $\mathcal{V}$, we randomly select a subset of nodes to drop, $\tilde{\mathcal{V}}\subset\mathcal{V}$, and denote the remaining nodes by $\mathcal{V}_o=\mathcal{V}\setminus\tilde{\mathcal{V}}$. The encoder then operates on the induced subgraph formed by removing dropped nodes and their incident edges, producing embeddings for the remaining nodes:
\begin{equation*}
\mZ_{\mathcal{V}_o}=f_{\text{enc}}(\mX_{\mathcal{V}_o},\, \mA_{\mathcal{V}_o,\mathcal{V}_o}),
\qquad
\tilde{\mZ}_i=
\begin{cases}
\mZ_i, & i\in\mathcal{V}_o,\\
\mathbf{0}, & i\in\tilde {\mathcal{V}}.
\end{cases}
\end{equation*}

For reconstruction, we feed the decoder the full graph structure with a node representation matrix $\tilde{\mZ}\in\mathbb{R}^{N\times D}$, which contains the encoded representations for retained nodes and the zero vectors for dropped nodes. We optimize the autoencoder \jm{$f_\text{enc}$ and $f_\text{dec}$} using a cosine reconstruction loss over the dropped nodes:
\[
\mathcal{L}_{\text{recon}}
=
\frac{1}{|\tilde {\mathcal{V}}|} {\textstyle \sum_{i\in\tilde {\mathcal{V}}}}
\Bigl(
1-\cos\bigl(f_{\text{dec}}(\tilde{\mZ},\mA)_i,\; \mX_i\bigr)
\Bigr).
\]

\mypar{Node4All.}
Finally, we introduce \emph{Node4All}, a node representation learner that integrates the three components: Channel Graph Transformer (CGT; \Cref{sec:cgt}), synthetic graph generation (\Cref{subsec:synthetic}), and masked autoencoding (\Cref{subsec:mae_dropnode}), as illustrated in \Cref{fig:node4all}. 
Node4All uses CGT as both the encoder $f_{\text{enc}}$ and decoder $f_{\text{dec}}$, enabling it to process graphs with arbitrary feature spaces. During training, a synthetic graph generator produces diverse graphs, and the model is trained via masked autoencoding by reconstructing dropped node features from their surrounding graph context. 
By repeatedly performing this reconstruction across heterogeneous graph structures and feature spaces, Node4All learns representations that capture generalizable patterns rather than dataset-specific ones, allowing node representation learning to generalize beyond individual datasets toward arbitrary graph datasets.

\section{Related Works}
Node4All is a node representation learner that shares the goal with graph representation learning methods, which we review in detail in \Cref{sec:introduction}. Here, we focus on two closely related lines of work: (1) \acp{GT}, which are architecturally related, and (2) \acp{GFM}, which similarly aim to generalize across multiple graph datasets within a single model.
 
\subsection{Graph Transformers}
\acp{GT} are a class of architectures that model graph-structured data using Transformers. Many \acp{GT} apply self-attention across nodes and augment attention with topology-aware positional encodings~\cite{dwivedi2020generalization,ying2021transformers}. While effective, their attention mechanism incurs quadratic complexity with respect to the number of nodes. To improve scalability, several \acp{GT} instead represent each node as a sequence of topology-aware tokens and apply self-attention within the sequence. For example, NAGphormer constructs tokens by aggregating features from different hop distances~\cite{chen2023nagphormer}, VCR-Graphormer builds token sequences using informative context nodes identified through personalized PageRank and virtual connections~\cite{fu2024vcrgraphormer}, and PolyFormer constructs tokens using polynomial graph filters that capture structural information at multiple scales~\cite{ma2024polyformer}.

While CGT also applies self-attention over node-specific token sequences, it differs fundamentally in how tokens are constructed. Existing methods operate on node feature vectors, where each token is derived from a rich multi-dimensional representation. In contrast, CGT operates on channel graphs, where each token originates from a scalar-valued feature channel. Applying the same aggregation operator to all token coordinates would therefore yield highly redundant representations, as every coordinate encodes the same signal at the same neighborhood scale. CGT addresses this issue through coordinate-wise multi-hop tokenization, assigning a distinct receptive field to each coordinate. Consequently, each token captures structural information from multiple neighborhood scales, yielding a richer representation of the channel signal.

\subsection{Graph Foundation Models}
We categorize existing \acp{GFM} approaches into three directions based on how broadly they generalize across datasets.

\mypar{Text-space \acp{GFM}}
The first direction unifies multiple graph datasets by operating in a common text space.
These methods leverage pretrained \acp{LLM} to embed raw text features~\cite{liu2024one, chen2025autogfm} or to make predictions directly within them~\cite{tang2024graphgpt, ye2024language, chen2024llaga}.
They have shown promising results in handling task heterogeneity, co-training on multiple datasets, transfer learning, and few-shot learning~\cite{liu2024one, chen2024llaga, he2025unigraph, fang2025uniglm}.
However, many graph datasets contain numerical or categorical attributes, and converting them into text is often infeasible or relies on dataset-specific templates~\cite{chen2024text}.

\mypar{Feature-aligned \acp{GFM}}
The second direction extends generalization to graphs with non-textual attributes.
To address feature heterogeneity across datasets (\textbf{L1}), these approaches project node features into a shared representation space via singular value decomposition with alignment modules~\cite{yu2025samgpt, zhao2024gcope} or learnable patching mechanisms~\cite{sun2025patchnet}, and have shown promising results in co-training and few-shot adaptation.
However, these methods typically require per-dataset fine-tuning when adapting to new graphs and are less effective when sufficient labels are available~\cite{sun2025patchnet, yu2025samgpt, zhao2024gcope}.

\mypar{Fully Inductive \acp{GFM}}
The third direction studies fully inductive \acp{GFM}, which operate on arbitrary graph datasets without training on the target data~\cite{zhaofully}.
GraphAny achieves full inductiveness by attending over a fixed set of least-mean-square linear predictors derived from features propagated at different scales.
TS-\acp{GNN}~\cite{finkelshtein2025equivarianceeverywhere} further generalizes this prediction framework by enforcing equivariance to node and feature permutations and invariance to labels.
NodePFN~\cite{anonymous2026learning} pursues a related goal by extending tabular foundation models, such as TabPFN~\cite{hollmann2022tabpfn, hollmann2025accurate}, to graph-structured inputs.
\dooho{
More recently, Recurrent GVT (RGVT)~\cite{lee2026viewspace} achieves cross-graph transfer by learning over a graph-filter-induced view space rather than the original node feature space.
}

\mypar{Relationship with \acp{GFM}}
Node4All differs from existing \acp{GFM} in two key aspects.
\dooho{
First, we target full inductive applicability across arbitrary graph datasets without any dataset-specific modification.
This is challenging for text-space \acp{GFM} due to their reliance on textual attributes and for feature-aligned \acp{GFM} due to their dependence on per-dataset fine-tuning.
While NodePFN and RGVT do not require training on the target dataset, they still rely on per-dataset preprocessing and recurrent depth selection, respectively.
}
Second, unlike most \acp{GFM} that conduct task-specific predictions, Node4All creates reusable representations. 
This enables (i) a single inference to produce embeddings that can be cached and reused across different label spaces, and (ii) flexible choices of downstream predictors.
In our experiments, we show that both one-shot settings of feature-aligned \acp{GFM} and in-context prediction of fully inductive \acp{GFM} can be realized, and often improved upon, by selecting an appropriate predictor on top of Node4All representations.

\begin{figure}[t]
\centering
\includegraphics[width=\linewidth]{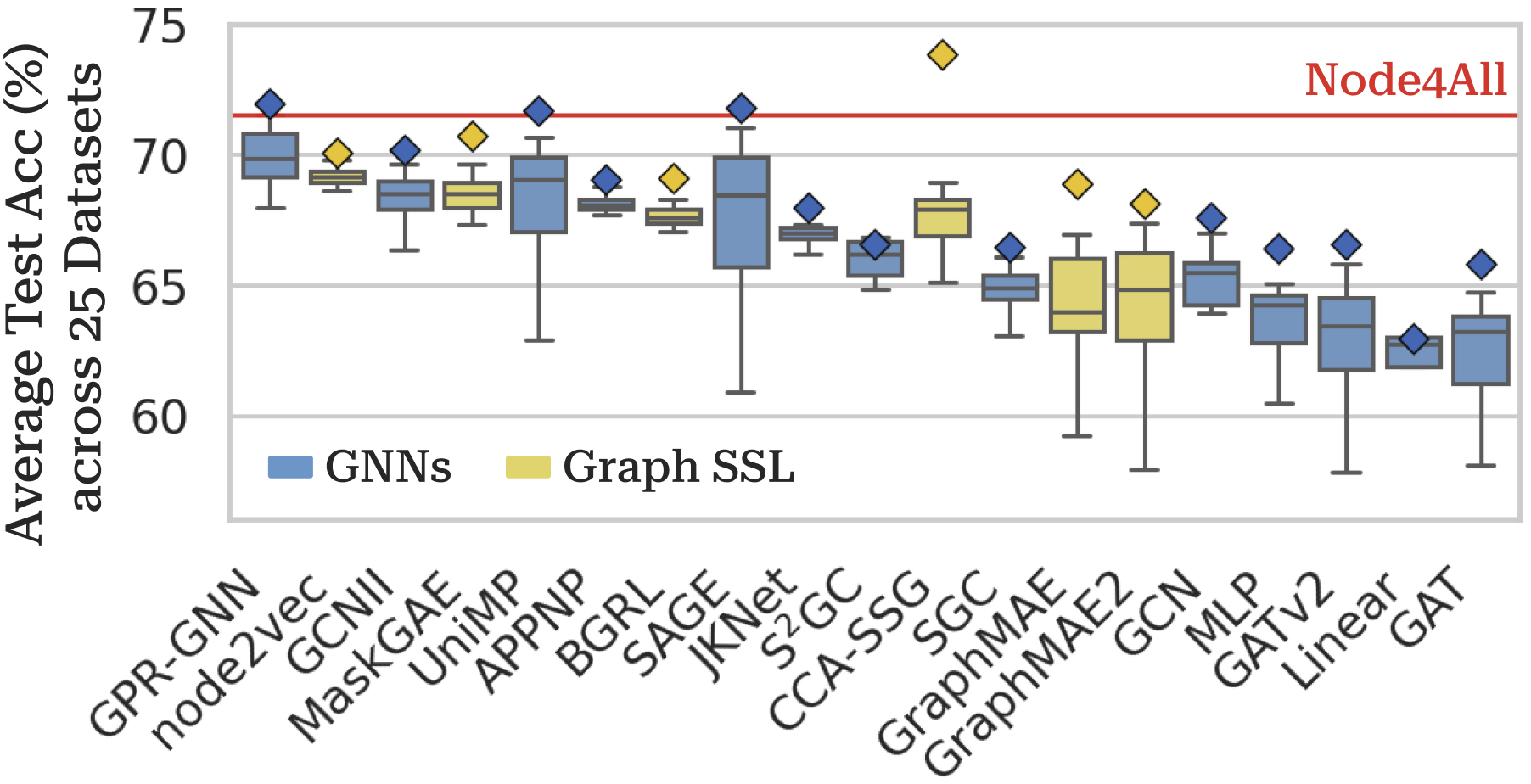}
\vspace{-6mm}
\caption{Box plots show the distribution of average test accuracies (\%) across hyperparameter configurations for each method. 
Diamonds denote the performance of each method under Tuned-HP. The red line is Node4All, which exhibits no variation since it requires no dataset-specific tuning.
}
\vspace{-2mm}
\label{fig:hp_box_plot}
\Description{Autoencoder}
\end{figure}
\section{Main Experiments}

We evaluate the practical effectiveness of Node4All through a series of experiments comparing it with existing methods across multiple settings.
Specifically, we investigate the following questions:

\begin{enumerate}[label=\textbf{Q\arabic*.}, leftmargin=2.5em,  topsep=2pt, parsep=0pt, itemsep=0pt ]
    \item Can a single fixed model (Node4All) make competitive performance compared to per-dataset optimized methods?
    \item How effective are Node4All representations in one-shot learning compared to feature-aligned \acp{GFM}?
    \item How effective are Node4All representations in in-context learning compared to fully inductive \acp{GFM}?
\end{enumerate}

\subsection{Against Per-Dataset Optimized Methods}
\label{subsec:main_exp}
Node4All offers clear practical advantages, including
\emph{(i)} applicability to arbitrary graph datasets and
\emph{(ii)} the elimination of dataset-specific training and hyperparameter tuning.
However, these benefits are only meaningful if performance remains competitive.
If too much accuracy is sacrificed, users may instead prefer the cost of per-dataset optimization.
We therefore evaluate whether a single fixed Node4All model can compete with conventional methods that are trained and optimized separately for each dataset.

\mypar{Setup.}
We compare Node4All against 21 baselines, including linear and \acp{MLP} predictors, node2vec~\cite{grover2016node2vec} \jm{with concatenated initial features,} 12 supervised \acp{GNN}, and 6 graph \acp{SSL} methods.
The supervised baselines include GCN~\cite{kipf2017semisupervised}, 
GraphSAGE~\cite{hamilton2017inductive},
APPNP~\cite{gasteiger2018combining}, 
JKNet~\cite{xu2018jknet}, 
GAT~\cite{velikovi2018graph}, 
SGC~\cite{wu2019simplifying}, 
GCNII~\cite{chen2020simple}, GATv2~\cite{brody2021attentive}, 
S$^2$GC~\cite{zhu2021simple}, UniMP~\cite{shi2021masked},  and GPR-GNN~\cite{chienadaptive}.
The self-supervised baselines include contrastive approaches (BGRL~\cite{thakoor2021bootstrapped} and CCA-SSG~\cite{zhang2021canonical}) and generative methods (GAE~\cite{kipf2016variational}, GraphMAE~\cite{hou2022graphmae}, GraphMAE2~\cite{hou2023graphmae2}, and MaskGAE~\cite{li2023s}).
For all representation learning methods, including Node4All, we use the same one-hidden-layer MLP predictor to isolate the quality of the learned representations.

We use 25 node classification benchmarks covering diverse domains, feature specifications, and structural properties, with label homophily ranging from 0.06 to 0.84 and average degree from 4.03 to 89.28.
Dataset details are provided in Appendix~\ref{appendix:dataset_details}.
Public data splits are used when available; otherwise, following~\cite{zhaofully}, we sample 20 nodes per class for training and evenly split the remaining nodes into validation and test sets.
\jm{For datasets with random splits, we select hyperparameters based on 5 independent splits, while test accuracy is reported as the mean and the standard deviation over a broader range of 10 splits.}

\mypar{Results.}
Due to the large number of benchmarks and baselines, we report a summary in Figure~\ref{fig:hp_box_plot}; full results are available on the project website.\footnote{\url{https://github.com/dooho00/node4all/blob/main/docs/full_exp_results.pdf}}
We evaluate baselines under two settings, \textbf{Fixed-HP} and \textbf{Tuned-HP}, to disentangle the effects of \emph{(i)} per-dataset training and \emph{(ii)} per-dataset hyperparameter tuning.
All baselines are trained separately on each target dataset, whereas Node4All is pretrained once on synthetic graphs and applied to all datasets as is, without \emph{any dataset-specific training or hyperparameter tuning.}

\jm{In \textbf{Fixed-HP}, the results are shown as the box plots in Figure \ref{fig:hp_box_plot}.}
Each method uses a single hyperparameter configuration selected to maximize its average validation accuracy across all datasets.
This setting isolates the effect of per-dataset training and tests whether a method admits a generally effective configuration that transfers across datasets.
Under this setting, Node4All achieves the highest average accuracy among all 21 baselines.
At the level of individual datasets, this result indicates that the knowledge captured through synthetic pretraining can match, and even surpass, the knowledge obtained through direct supervision or self-supervised learning on target graphs.
At the cross-dataset level, the results show that a single Node4All model can generalize across diverse graph datasets by learning from large collections of synthetic graphs, outperforming methods trained separately on each dataset.
Overall, these findings suggest that when extensive per-dataset hyperparameter tuning is not affordable, Node4All provides a strong and reliable alternative to conventional per-dataset training.

\jm{In \textbf{Tuned-HP}, the results are visualized as the diamond markers in Figure \ref{fig:hp_box_plot}.}
Hyperparameters are selected independently for each dataset, allowing each method to benefit from dataset-specific hyperparameter choices. Since Node4All does not perform any per-dataset hyperparameter tuning, its performance remains the same as in \textbf{Fixed-HP.}
Under this setting, graph \acp{SSL} methods improve their average accuracy by 1.92\%, and supervised \acp{GNN} baselines by 1.11\%, while Node4All’s rank shifts from 1st to 5th among 21 baselines.
This highlights a key property of existing methods: their strongest results typically emerge only after extensive per-dataset tuning and training, and a single broad configuration is often insufficient.
The cost of this specialization is reflected in strong hyperparameter sensitivity, as illustrated in \Cref{fig:hp_box_plot}, which shows substantial variation in average performance across hyperparameter configurations. In contrast, Node4All employs a single fixed model, requires no per-dataset optimization, and yet remains competitive even against heavily tuned baseline methods, making it a more robust and practically favorable choice.

\begin{table}[t]
\centering
\caption{One-shot learning performance (\%) compared with feature-aligned \acp{GFM} across 5 datasets.}
\label{tab:one-shot}
\vspace{0mm}
\resizebox{\linewidth}{!}{%
\begin{tabular}{lccccc}
\toprule
\textbf{Method} &
\text{Cora} &
\text{Citeseer} &
\text{Pubmed} &
\text{AmzPhoto} &
\text{AmzComp} \\
\midrule

{GCN} & $29.53_{\pm7.56}$ & $26.29_{\pm6.50}$ & $23.32_{\pm11.56}$ & $26.96_{\pm12.94}$ & $24.40_{\pm5.62}$ \\

{GraphPrompt} & $28.26_{\pm12.68}$ & $32.51_{\pm8.73}$ & $47.47_{\pm9.15}$ & $48.11_{\pm9.89}$ & $42.82_{\pm11.67}$ \\

{GPF} & $32.17_{\pm6.56}$ & $36.79_{\pm7.70}$ & $41.28_{\pm8.14}$ & $47.47_{\pm8.19}$ & $35.75_{\pm7.12}$ \\

{Hassani} & $33.35_{\pm6.93}$ & $33.66_{\pm7.24}$ & $39.87_{\pm8.16}$ & $48.48_{\pm7.07}$ & $39.99_{\pm7.91}$ \\

{GCOPE} & $35.62_{\pm11.93}$
& $38.33_{\pm9.28}$ & $45.38_{\pm9.87}$ & $52.87_{\pm9.19}$ & $45.65_{\pm10.69}$ \\

{SAMGPT} & $\underline{47.80_{\pm11.88}}$
& $36.38_{\pm9.10}$ & $\underline{50.65_{\pm8.93}}$ & $\underline{58.71_{\pm8.69}}$ & $\underline{48.22_{\pm8.17}}$ \\

\midrule
{Ridge Reg.} 
& $30.54_{\pm10.63}$
& $\underline{39.83_{\pm8.15}}$ & \textcolor{Red}{$\mathbf{52.42_{\pm8.09}}$} & $52.42_{\pm8.09}$ & $38.58_{\pm8.87}$ \\

\textbf{+ Node4All} & \textcolor{Red}{$\mathbf{51.79_{\pm12.99}}$}
& \textcolor{Red}{$\mathbf{44.08_{\pm8.94}}$} & $38.76_{\pm10.51}$ & \textcolor{Red}{$\mathbf{58.97_{\pm9.69}}$} & \textcolor{Red}{$\mathbf{49.72_{\pm10.64}}$} \\

\bottomrule
\end{tabular}
\vspace{-6mm}
}

\end{table}

\subsection{One-Shot Learning}

One-shot learning is a common evaluation setting for \acp{GFM}, where only a single labeled example per class is provided.
It is designed to assess whether knowledge pretrained on other datasets can be effectively adapted to a new graph dataset under minimal supervision.
We use this setting to evaluate the effectiveness of Node4All representations under extremely limited supervision and compare them against recent feature-aligned \acp{GFM}, which explicitly aim to support cross-dataset adaptation in this setting.

\mypar{Setup.}
We follow the evaluation protocol of SAMGPT~\cite{yu2025samgpt} and use their reported numbers on five datasets.
Baselines include supervised GCN~\cite{kipf2017semisupervised}, feature-aligned \acp{GFM} (GCOPE~\cite{zhao2024gcope}, SAMGPT~\cite{yu2025samgpt}), graph cross-domain adaptation model (Hassani~\cite{hassani2022cross}), and graph prompting frameworks (GPF~\cite{fang2023universal}, GraphPrompt~\cite{liu2023graphprompt}).
Node4All is applied to all datasets without any dataset-specific modification, as in the main experiment.
To adapt the node representations to the labels with only one-shot supervision, we first reduce their dimensionality to 32 using truncated SVD and then train a lightweight ridge regressor with a regularization strength $\alpha=10.0$.
This procedure is applied uniformly across all datasets.

\mypar{Results.}
We report the accuracy in \Cref{tab:one-shot}.
Node4All outperforms GCOPE and SAMGPT, which are feature-aligned \acp{GFM}, on 4 out of 5 datasets, as well as other baselines.
These results demonstrate that Node4All remain effective even under extreme supervision scarcity.
Importantly, while the \ac{GFM} baselines rely on co-training across multiple real-world datasets, Node4All is trained solely on synthetic graphs.
These findings support our claim that training on well-designed synthetic graphs can generalize more effectively than training on a limited collection of real-world datasets.

\begin{table*}[t]
\centering
\caption{In-context learning performance (\%) compared with fully inductive \acp{GFM} across 11 datasets.
The last column reports the average rank (lower is better), with \textcolor{Red}{\textbf{best}} and \underline{second-best} results highlighted for each dataset.}
\label{tab:in-context-learning}
\vspace{-2mm}
\resizebox{\linewidth}{!}{%
\begin{tabular}{lccccccccccc|c}
\toprule
\textbf{Method} & \text{Actor} &
\makecell{Amazon\\Photo} & \makecell{Amazon\\Computer} &
 \text{Cornell} & \text{DBLP}&
\text{Deezer} & \makecell{Mine-\\sweeper} & \text{Pubmed} & \text{Tolokers} & \text{Texas} & \text{Wisconsin} & \makecell{\textbf{Avg.}\\\textbf{Rank}}\\
\midrule

{GraphAny} & $29.51_{\pm0.55}$ & $\underline{90.64_{\pm0.82}}$ & $\underline{83.04_{\pm1.24}}$  & $66.49_{\pm1.48}$ & $66.49_{\pm1.48}$ & $52.13_{\pm3.02}$ & $80.46_{\pm0.15}$& $77.46_{\pm0.30}$ & $78.24_{\pm0.03}$ & $73.52_{\pm2.96 }$& $71.77_{\pm5.98}$ & 3.90 \\

{TS-GAT} & $28.80_{\pm2.12}$ & $89.66_{\pm1.23}$ & $80.23_{\pm2.44}$  & $71.35_{\pm4.52}$ & $64.30_{\pm3.75}$ & $52.88_{\pm2.86}$ & \textcolor{Red}{$\mathbf{80.72_{\pm0.78}}$}& $75.46_{\pm0.86}$ & $78.01_{\pm0.23}$ & $ 74.05_{\pm4.10}$& $ 65.88_{\pm9.86}$ & 4.27 \\

{TS-Mean} & $28.09_{\pm0.93}$ & $ 90.18_{\pm1.30}$ & $ 81.37_{\pm1.25}$  & $68.65_{\pm2.42}$ & $66.42_{\pm3.65}$ & $52.31_{\pm2.52}$ & 
$\underline{80.68_{\pm0.38}}$& $74.98_{\pm0.56}$ & $78.12_{\pm0.09}$ & $73.51_{\pm4.01}$& $61.18_{\pm11.38}$ & 4.63 \\

{NodePFN} & $\underline{32.99_{\pm1.09}}$ & $90.53_{\pm0.13}$ & $81.42_{\pm0.48}$  & $\underline{71.89_{\pm2.76}}$ & $\underline{74.71_{\pm0.39}}$ & $53.45_{\pm0.65}$ & $80.66_{\pm0.25}$& $\underline{78.00_{\pm0.24}}$ & $\underline{78.61_{\pm0.06}}$ & $76.22_{\pm7.53}$& $79.22_{\pm6.97}$ & $\underline{2.54}$ \\

\midrule

{TabPFN} & $32.93_{\pm1.71}$ & $80.78_{\pm1.63}$ & $67.02_{\pm2.00}$ & \textcolor{Red}{$\mathbf{74.05_{\pm5.57}}$} & $53.46_{\pm2.53}$ & \textcolor{Red}{$\mathbf{55.10_{\pm2.48}}$} & $80.00_{\pm0.00}$& $72.20_{\pm0.00}$ & $78.36_{\pm0.24}$ & \textcolor{Red}{$\mathbf{82.08_{\pm5.16}}$}& \textcolor{Red}{$\mathbf{86.08_{\pm3.97}}$} & 3.63\\

\textbf{+ Node4All} & \textcolor{Red}{$\mathbf{33.18_{\pm1.15}}$} & \textcolor{Red}{$\mathbf{91.70_{\pm0.65}}$} & \textcolor{Red}{$\mathbf{83.59_{\pm0.73}}$} & $66.76_{\pm7.05}$ & 
\textcolor{Red}{$\mathbf{77.67_{\pm1.54}}$} & $\underline{55.06_{\pm1.30}}$&
$80.36_{\pm0.36}$& 
\textcolor{Red}{$\mathbf{80.80_{\pm0.00}}$}& 
\textcolor{Red}{$\mathbf{81.15_{\pm0.61}}$} & $\underline{76.88_{\pm5.40}}$ & $\underline{82.35_{\pm3.62}}$ & \textcolor{Red}{$\mathbf{2.00}$}\\

\bottomrule
\end{tabular}
}
\end{table*}

\subsection{In-Context Learning}

In-context learning refers to making predictions without updating model parameters, instead using a small set of labeled examples as context.
For \acp{GFM}, this setting evaluates whether pretrained knowledge from other datasets can be effectively applied to unseen datasets by conditioning on the provided context.
We use this evaluation to assess whether Node4All representations support in-context adaptation and compare them against fully inductive \acp{GFM}, which are designed explicitly for in-context learning on graphs.

\mypar{Setup.}
We compare against three recent fully inductive \acp{GFM}: GraphAny~\cite{zhaofully}, TSNet~\cite{finkelshtein2025equivarianceeverywhere}, and NodePFN~\cite{anonymous2026learning}.
We follow the shared evaluation protocol when evaluating Node4All, while baseline results are taken directly from the original papers.

Since Node4All does not natively support in-context learning, we pair it with TabPFN~\cite{hollmann2022tabpfn, hollmann2025accurate}, a tabular foundation model, as an in-context predictor.
Node representations produced by Node4All are provided to TabPFN together with the labeled nodes as context.
This Node4All-TabPFN pipeline enables in-context predictions on arbitrary graphs in a single inference pass, without any training or dataset-specific modification, thereby forming a fully inductive \ac{GFM}.
Due to current limitations of TabPFN on feature dimensionality and the number of context samples, we evaluate this setup on 11 datasets that satisfy these constraints.

\mypar{Results.}
We report the results in \Cref{tab:in-context-learning}.
Node4All achieves the best average rank of 2.0 and the top performance on 6 datasets.
These results show that Node4All representations are highly effective for in-context learning, transforming the structure-agnostic TabPFN into a graph foundation model and outperforming state-of-the-art fully inductive \acp{GFM}.
Notably, NodePFN, the second-best method, requires per-dataset data preprocessing, including the PCA dimensionality and propagation depth, making our strong performance particularly remarkable given that it requires no dataset-specific modification.
\jm{Moreover, Node4All with TabPFN generally outperforms Node4All with the MLP predictor reported in the main experiment, highlighting the promise of this combination.}

\section{Additional Analysis}

In this section, we present further analyses of Node4All, examining its computational complexity, representational behavior via visualization, and the effects of its core design choices.

\subsection{Complexity Analysis}
\label{subsec:scalability}

\mypar{Theoretical.}
The computational cost of Node4All is dominated by CGT, which includes multi-hop tokenization followed by Transformer pooling.
Let $d$ be the token dimension, $(L+1)$ the token sequence length, and $R$ the number of layers.
The computational cost decomposes into multi-hop tokenization
$\mathcal{O}(L(|E|d^2 + Nd^2))$,
Transformer attention $\mathcal{O}(RN(L+1)^2d)$,
and feed-forward layers $\mathcal{O}(RN(L+1)d^2)$, all of which are linear with the number of nodes $N$ and edges $|E|$.
Since this cost is incurred per channel graph, CGT scales linearly with the input feature dimensionality $F$.

\mypar{Empirical.}
On an NVIDIA RTX A6000, training Node4All takes about 10 minutes for 500 iterations over 500 synthetic graphs, and inference across 25 benchmark graphs takes 672 seconds.
The main bottleneck is that CGT processes all $F$ feature channels separately, without compressing them into a hidden dimension $H$ as in standard \acp{GNN}, leading to higher inference time on datasets with large feature dimensionality, such as Coauthor Physics (238\,s), FullCora (103\,s), and BlogCatalog (90\,s) with even $H > 8000$.
Nevertheless, runtime scales linearly with $F$ and remains practical, with inference time of several minutes.
Moreover, Node4All offers two efficiency advantages:
\emph{(i)} representations can be cached and reused, and
\emph{(ii)} no dataset-specific hyperparameter search is required, which often dominates the overall cost of existing methods.

\begin{figure}[t]
\centering

\includegraphics[width=\linewidth]{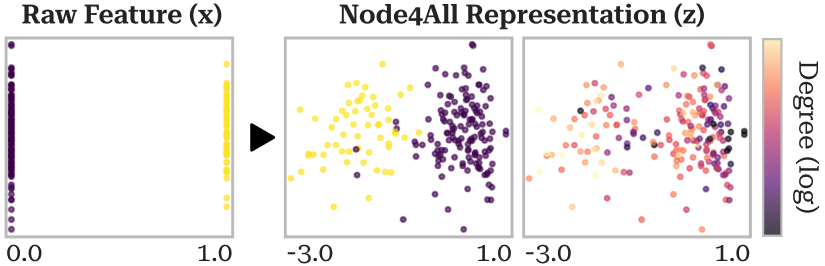}
\vspace{-4mm}
\caption{
Scatter plots of raw features (left) and Node4All representations (middle, right) for the first feature channel of Amazon Photo, with nodes colored by degree (right).
}
\label{fig:channel-wise}
\vspace{-6mm}
\Description{Autoencoder}
\end{figure}

\subsection{Visualization}
\label{subsec:viz}

While recent fully inductive \acp{GFM} can operate across arbitrary datasets, it remains unclear what kind of knowledge enables such broad generalization.
In this section, we use visualizations to examine how Node4All behaves and to explain why this mechanism supports generalization across datasets.

\mypar{Channel-wise Processing.}
Node4All processes each feature channel independently, which allows us to directly inspect how an individual channel changes before and after graph-aware encoding.
In \Cref{fig:channel-wise}, each node appears as a point, with its scalar feature value on the horizontal axis, while the vertical axis is included for visualization.
For binary features, many nodes initially share identical values (0 or 1; left).
After encoding, Node4All maps these nodes to distinct representations while preserving the original two-group separation (middle), and the resulting ordering further correlates with node degree (right).
These results show how Node4All injects structural information into each node group that shares their features.
We observe the same behavior consistently for categorical and continuous features, as shown in Appendix~\ref{appendix:visualize}.

\mypar{Representation Space.}
We next visualize the full node representations using t-SNE~\cite{maaten2008visualizing}.
When colored by the first principal component of the raw features, the overall feature-driven structure is largely preserved in the learned representations.
When colored by node degree, clear structural patterns emerge in the Node4All representation space that are not visible in the raw features.
Together, these observations suggest that the channel-wise structural effects introduced by Node4All accumulate coherently in the final representation space.
Importantly, this mechanism---distinguishing nodes based on graph structure while preserving original feature information---relies on generic structural cues rather than dataset-specific semantics, which helps explain why Node4All generalizes effectively across diverse graph datasets.

\begin{figure}[t]
\centering
\vspace{-1mm}
\includegraphics[width=\linewidth]{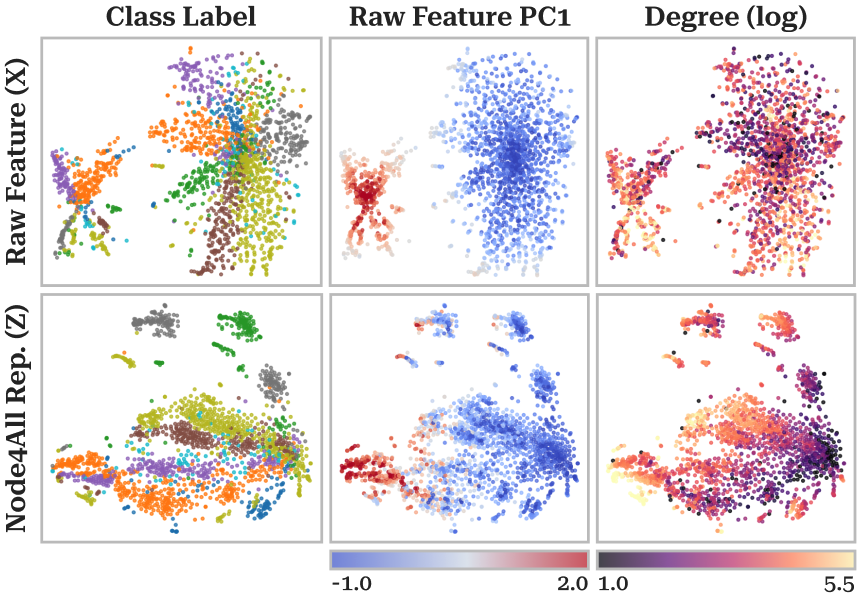}
\vspace{-6mm}
\caption{
t-SNE visualizations of raw features and Node4All representations on Amazon Photo, colored by class labels (left), raw-feature PC1 (middle), and node degree (right).
}
\label{fig:tsne}
\vspace{-2mm}
\Description{Autoencoder}
\end{figure}

\subsection{Ablation Study}
\label{subsec:ablation}
Due to space limitations, detailed ablation results are provided in Appendix~\ref{appendix:ablations}. 
Replacing the backbone architecture, CGT with a GCN reduces average accuracy to 37.80\%, far below the Node4All model (71.55\%), highlighting the necessity of CGT for channel-wise graph processing. 
Training on real-world datasets instead of synthetic graphs yields lower performance, with OGBN-Arxiv achieving the best average accuracy (69.94\%), comparable to mid-ranked baselines (9th of 21). 
This result indicates that Node4All’s architecture and masked autoencoding form an effective framework for cross-dataset representation learning, while synthetic graph pretraining drives strong generalization to arbitrary graph datasets.
\section{Discussion}

In this section, we discuss the inductive bias of Node4All, as well as its current limitations and directions for future research.

\subsection{Inductive Bias}

CGT processes each feature channel independently, enabling it to handle arbitrary feature dimensionalities. Beyond this flexibility, the channel-wise design promotes transferability across datasets by avoiding the explicit modeling of cross-feature interactions. While individual feature channels often carry transferable semantic information, the ways in which features interact to predict labels can vary substantially across datasets and label spaces. Capturing such interactions during pretraining may therefore cause representations to become specialized to particular datasets and label spaces, limiting their ability to generalize across datasets. Consistent with this view, explicitly modeling cross-feature interactions through linear self-attention reduces the average performance of Node4All across the 25 benchmark datasets to 70.78\%. This result supports our design choice of preserving channel identities while learning structural patterns during pretraining and delegating the modeling of cross-feature interactions to downstream predictors.

Our learning framework further reinforces this inductive bias. The feature reconstruction objective encourages the preservation of feature semantics, while training on diverse synthetic graphs promotes generalization across varying graph structures and feature spaces. Together, these design choices enable Node4All to learn transferable representations across diverse graph datasets.

\subsection{Limitations and Future Work}

\mypar{Synthetic graph generation.}
The primary goal of this work is to establish the feasibility of dataset-agnostic graph representation learning. Consequently, our synthetic graph generation framework remains relatively simple. An important direction for future work is to explore richer graph generation strategies that capture a broader range of structural patterns. For example, stochastic block model-based generators~\cite{karrer2011stochastic} could provide more realistic community structures, while motif-aware generation strategies~\cite{benson2016higher} could introduce higher-order patterns commonly observed in real-world graphs. By generating graphs whose structural properties more closely resemble those of real-world datasets, such approaches may yield more transferable and robust node representations.

\mypar{Beyond node classification.}
This work focuses exclusively on node classification tasks. Extending dataset-agnostic representation learning to graph-level tasks involving small graphs, such as molecular property prediction, as well as edge-level tasks, is an important direction toward more general graph representation learners. Supporting these prediction settings may require architectural adaptations and pretraining objectives tailored to different prediction levels. For example, edge-level tasks may benefit from objectives that explicitly capture pairwise relationships and local connectivity patterns, while graph-level tasks may require mechanisms that better aggregate graph-wide information.
\section{Conclusion}

In this work, we introduce Node4All, a single node representation learner for arbitrary real-world graph datasets, without requiring any dataset-specific training or hyperparameter tuning.
Node4All differs from existing methods in both its architecture and learning framework. Architecturally, we propose the Channel Graph Transformer (CGT), a novel channel-wise graph processing architecture that accommodates heterogeneous feature dimensionalities across datasets within a single fixed model. In terms of learning, unlike conventional graph representation learning methods that are optimized separately for each target dataset, Node4All is trained once, entirely on synthetic graphs. 
Through carefully designed synthetic graph generation that spans a diverse range of graphs, Node4All learns representations that generalizable across a wide range of graph datasets. Empirically, a single Node4All model achieves competitive performance against per-dataset optimized methods and even outperforms recent graph foundation models in one-shot and in-context learning. Overall, Node4All advances node representation learning beyond dataset-bounds.


\begin{acks}
{\small 
This work was partly supported by the National Research Foundation of Korea (NRF) grant funded by the Korea government (MSIT) (RS-2024-00341425 and RS-2024-00406985), ``Advanced GPU Utilization Support Program'' funded by the Government of the Republic of Korea (Ministry of Science and ICT), and the New Faculty Startup Fund from Seoul National University.
Jaemin Yoo is the corresponding author.
}
\end{acks}

\bibliographystyle{ACM-Reference-Format}
\bibliography{sample-base}


\appendix

\section{Limitations of Random Node Assignment}
\label{appendix:lemma}
A key challenge in synthetic graph generation is producing node features that reflect the underlying graph structure. A natural approach is to assign features randomly and independently to each node. However, because this process ignores the graph topology, the resulting features contain little information about structural differences between neighboring nodes.

This limitation can be characterized spectrally. High-frequency graph components capture localized variations induced by the graph structure, yet random feature assignment contributes only limited energy to these components. The following lemma formalizes this observation and motivates our use of random propagation to inject graph-dependent structure into the generated features.

\begin{lemma}
\label{lem:small_highfreq_energy}
Let $\mX \in \mathbb{R}^{N \times F}$ be node features generated independently of the graph structure, and let $\mA \in \mathbb{R}^{N \times N}$ denote the adjacency matrix.
Then, in expectation, at most a fraction of $k / N$ of the total feature energy lies in any $k$-dimensional high-frequency eigenspace of the graph Laplacian.
\end{lemma}

\begin{proof}
Let $\mL=\mI-\mD^{-1/2}\mA\mD^{-1/2}$ and let $\mP_H$ be an orthogonal projector onto a $k$-dimensional high-frequency eigenspace of $\mL$.
We generate $\mX=\mS\mW$, where $\mW\in\mathbb{R}^{\bar F\times F}$ has orthonormal rows ($\mW\mW^\top=\mI_{\bar F}$), and
for each factor $j\in[\bar F]$ we sample
\[
\mu_j\sim\mathcal{N}(0,1),\qquad z_j\sim\mathcal{N}(0,1),\qquad \sigma_j = |z_j\mu_j|,
\]
then draw node-wise entries independently as
\[
S_{u,j}\mid(\mu_j,\sigma_j)\sim\mathcal{N}(\mu_j,\sigma_j^2),\qquad u\in[N].
\]
Since $\mW\mW^\top=\mI_{\bar F}$, we have $\mX\mX^\top=\mS\mS^\top$.
Moreover, for $u\neq v$,
\[
\mathbb{E}[S_{u,j}S_{v,j}]
=
\mathbb{E}[\mu_j^2],
\qquad
\mathbb{E}[S_{u,j}^2]
=
\mathbb{E}[\sigma_j^2]+\mathbb{E}[\mu_j^2],
\]
hence
\[
\mathbb{E}[\mX\mX^\top]
=
\mathbb{E}[\mS\mS^\top]
=
\bar F\,\mathbb{E}[\sigma^2]\,\mI_N
+
\bar F\,\mathbb{E}[\mu^2]\,\mathbf{1}\mathbf{1}^\top.
\]
Therefore,
\[
\mathbb{E}\bigl[\|\mP_H\mX\|_F^2\bigr]
=
\mathrm{tr}\!\bigl(\mP_H\,\mathbb{E}[\mX\mX^\top]\bigr)
=
\bar F\,\mathbb{E}[\sigma^2]\mathrm{tr}(\mP_H)
+
\bar F\,\mathbb{E}[\mu^2]\mathbf{1}^\top \mP_H \mathbf{1}.
\]
Using $\mathrm{tr}(\mP_H)=k$ and $0\preceq \mP_H \preceq \mI$ (so $\mathbf{1}^\top \mP_H \mathbf{1}\le \mathbf{1}^\top \mI \mathbf{1}=N$), we obtain
\[
\mathbb{E}\bigl[\|\mP_H\mX\|_F^2\bigr]
\le
\bar F\,\mathbb{E}[\sigma^2]\,k
+
\bar F\,\mathbb{E}[\mu^2]\,N.
\]
Also,
\[
\mathbb{E}\bigl[\|\mX\|_F^2\bigr]
=
\mathrm{tr}\!\bigl(\mathbb{E}[\mX\mX^\top]\bigr)
=
\bar F\,\mathbb{E}[\sigma^2]\,N
+
\bar F\,\mathbb{E}[\mu^2]\,N^2.
\]
Hence for $k\ge 1$,
\[
\frac{\mathbb{E}[\|\mP_H\mX\|_F^2]}{\mathbb{E}[\|\mX\|_F^2]}
\le
\frac{\bar F\,\mathbb{E}[\sigma^2]\,k + \bar F\,\mathbb{E}[\mu^2]\,N}
{\bar F\,\mathbb{E}[\sigma^2]\,N + \bar F\,\mathbb{E}[\mu^2]\,N^2}
\le
\frac{k}{N},
\]
which proves the claim.
\end{proof}

\section{Ablation Studies}
\label{appendix:ablations}
\begin{table}[t]
\caption{Ablation study on the CGT.}
\vspace{-3mm}
\resizebox{0.9\linewidth}{!}{%
\begin{tabular}{lc}
\toprule
\textbf{Modification} &
\textbf{Total Avg. (25)} \\
\midrule
Use GAT instead CGT & $37.80_{\pm1.29}$  \\
Use GCN instead CGT & $37.61_{\pm0.92}$  \\
\midrule
Use GAT instead Multi-hop Tokenization & $68.92_{\pm1.34}$  \\
Use GCN instead Multi-hop Tokenization & $69.85_{\pm1.48}$  \\
Use MLP instead Transformer Readout & $68.14_{\pm2.34}$  \\
\midrule
\textbf{Node4All} & $\mathbf{71.55_{\pm1.38}}$\\
\bottomrule
\end{tabular}}
\vspace{-2mm}
\label{tab:ablation_architecture}
\end{table}

\begin{table}[t]
\caption{Average test accuracy (\%) across 25 datasets for CGT trained on real-world datasets with masked autoencoding.}
\vspace{-3mm}
\resizebox{0.9\linewidth}{!}{%
\begin{tabular}{lccc}
\toprule
\textbf{Dataset} &
\textbf{\# Nodes} &
\textbf{Total Avg. (25)} &
\makecell{\textbf{Rank @} \\\textbf{Baselines}}\\
\midrule
Wisconsin  & 251 & $60.39_{\pm3.58}$  & 20\\
Cora & 2,708 & $61.79_{\pm3.80}$  & 20 \\
Roman Empire  & 22,662 & $55.32_{\pm1.78}$  & 20 \\
Questions  & 48,921 & $69.58_{\pm1.60}$  & 9\\
OGBN-Arxiv  & 169,343 & $69.94_{\pm1.48}$  & 9\\
OGBN-Products  & 2,449,029	 & $68.82_{\pm1.36}$  & 10\\
\midrule
\textbf{Node4All} & - & $\mathbf{71.55_{\pm1.38}}$ & 5\\
\bottomrule
\end{tabular}}
\vspace{-2mm}
\label{tab:ablation_datasets}
\end{table}
While we provide a summary of the ablation study in \Cref{subsec:ablation}, we present detailed results and analyses here.

\mypar{Channel Graph Transformer (CGT).}
CGT is designed to process a \emph{channel graph}, where each node is associated with a single scalar feature.
As a baseline, we apply standard \acp{GNN} architectures (GCN and GAT) by lifting scalar features to a $d$-dimensional space, performing message passing, and projecting them back to scalars.
As shown in \Cref{tab:ablation_architecture}, this results in severe performance degradation, with average accuracies of 37.80\% and 37.61\%, respectively, even below a linear classifier (62.92\%).
These results confirm that learning on channel graphs requires a specialized architecture such as CGT.

CGT consists of two components: multi-hop tokenization and Transformer-based token readout (\Cref{sec:cgt}). 
Since removing either component makes the model ill-defined, we instead evaluate alternative designs. 
First, we replace dimension-wise tokenization with shared aggregation using GCN or GAT as token generators. 
This modification reduces performance to 69.85\% and 68.92\%, respectively, though the degradation is less severe than using \acp{GNN} backbones alone. 
Similarly, replacing Transformer-based aggregation with MLP pooling also degrades performance, but again to a lesser extent than using \acp{GNN} alone. 
Overall, these results highlight that the tokenization--aggregation paradigm is essential for learning on channel graphs, with the proposed CGT design achieving the best performance.

\mypar{Training on Real-world Datasets.}
While Node4All is trained on synthetic graphs to achieve broad generalization, training on real-world datasets is also possible. 
We evaluate CGT models trained on six real-world datasets (\Cref{tab:ablation_datasets}) and find that all of them underperform the synthetic-pretrained Node4All, highlighting the advantage of using synthetic graphs on pretraining. 

In details, CGT fails to learn meaningful representations on small datasets (e.g., Wisconsin, Cora, and Roman Empire), but achieves mid-ranked performance on larger datasets (e.g., Questions, OGBN-Arxiv, and OGBN-Products). 
However, larger size alone does not guarantee better results: despite being much larger than OGBN-Arxiv, OGBN-Products does not yield improved performance. 
These observations suggest that effective training depends more on structural and feature diversity than on dataset size alone.
Overall, these results show that pretraining on synthetic graph provides a controllable and effective way to induce diverse training signals.

\begin{figure}[t]
\centering
\includegraphics[width=\linewidth]{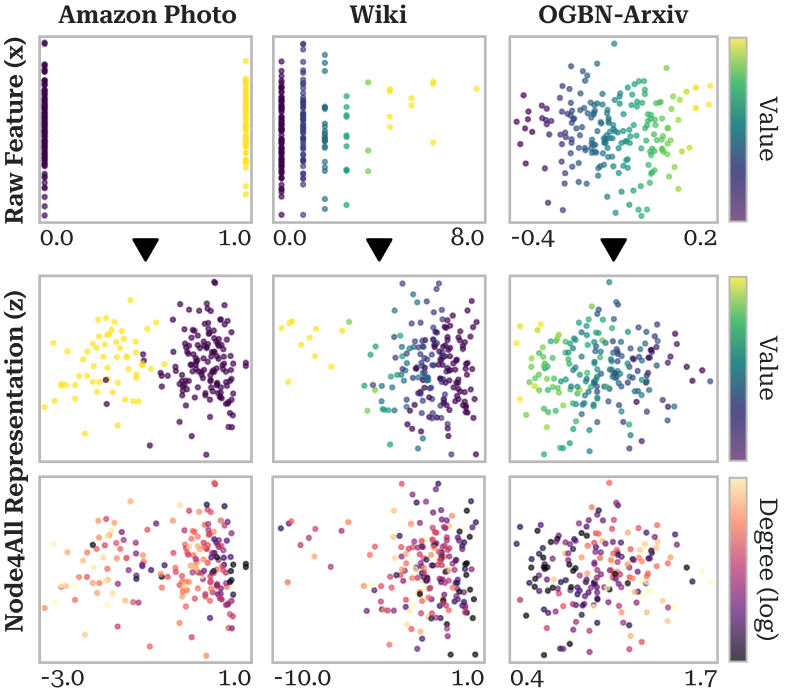}
\vspace{-6mm}
\caption{
Scatter plots of raw features (top) and Node4All representations (bottom) for the first feature channel on Amazon Photo, Wiki, and OGBN-Arxiv, representing binary, categorical, and continuous features.
}
\label{fig:channel-wise-extend}
\Description{Autoencoder}
\vspace{-4mm}
\end{figure}

\section{Additional Visualization}
As described in \Cref{subsec:viz}, we visualize how channel values change before and after processing by CGT. 
Here, we provide additional visualizations on datasets with different feature types—binary, categorical, and continuous—as shown in \Cref{fig:channel-wise-extend}.
\label{appendix:visualize}

\section{Dataset Details}
\label{appendix:dataset_details}
Our dataset selection and splitting strategy largely follows prior work~\citep{zhaofully}.
We obtain datasets from PyG~\citep{fey2019fast}, DGL~\citep{wang2019deep}, and the heterophilous graph datasets released by Yandex Research.\footnote{\url{https://github.com/yandex-research/heterophilous-graphs}}
All graphs are treated as undirected.
For the Chameleon and Squirrel datasets, we use the filtered versions provided in this repository, as the original datasets have been reported to contain issues~\citep{platonov2023a}.
Detailed statistics for each dataset are reported in project website.\footnote{\url{https://github.com/dooho00/node4all/blob/main/docs/dataset_stats.pdf}}


\end{document}